\documentclass[10pt,twocolumn,letterpaper]{article}

\usepackage{cvpr}      
\usepackage{amsmath, amssymb, amsthm}
\usepackage{algorithm}
\usepackage{algcompatible}
\usepackage{bbm}

\newtheorem{assumption}{Assumption}
\newtheorem{theorem}{Theorem}

\newcommand{\ours}{{FineMoLA}\xspace}
\definecolor{cvprblue}{rgb}{0.21,0.49,0.74}
\usepackage[pagebackref,breaklinks,colorlinks,allcolors=cvprblue]{hyperref}

\usepackage{listings}
\usepackage{xcolor}
\usepackage{hyperref}
\usepackage{mdframed}

\def\paperID{8} 
\def\confName{CVPR}
\def\confYear{2026}

\title{FineMoLA: Towards Fine-Grained Motion-Language Alignment from Clip-Level Supervision}

\author{
    Tongyan Wang$^{1}$\thanks{Equal contribution.}, 
    Zhengyuan Li$^{1}$\footnotemark[1], 
    Muhan Lin$^{1}$, 
    Shengyang Luo$^{1}$, 
    Yifan Shen$^{2}$, \\
    Aniket Bera$^{1}$,  
    Baijian Yang$^{1}$,  
    Yingjie Victor Chen$^{1}$ \\
    $^{1}$Purdue University \quad $^{2}$University of Illinois Urbana-Champaign \\
    {\tt\small \{wang5298, li5280, lin2265, luo525\}@purdue.edu,} 
    {\tt\small yifan26@illinois.edu} \\
    {\tt\small \{aniketbera, byang, victorchen\}@purdue.edu}
}

\begin{document}
\maketitle
\begin{abstract}
Text-conditioned human motion generation has made rapid progress with the emergence of large-scale motion--language datasets. However, even datasets with rich long-form descriptions typically provide supervision only at the clip level, without explicit temporal correspondence between motion frames and language. This limits fine-grained motion--text grounding and temporally precise generation. We propose \ours, a weakly supervised framework that learns fine-grained frame--phrase correspondence directly from clip-level annotations. Our method first segments long-form descriptions into action-bearing phrases, and then formulates motion--language alignment as an optimal transport problem, which naturally models many-to-many relations between motion frames and text under global constraints. With entropic regularization and Sinkhorn iterations, \ours efficiently infers pseudo frame-level alignments without human labeling. Experiments on SnapMoGen demonstrate that the learned alignments outperform baselines in motion--text grounding. Website: \url{https://hellosunnyworld.github.io/finemola/}.

\end{abstract}    

\section{Introduction}

Recent years have witnessed rapid progress in text-conditioned 3D human motion generation, driven by advances in diffusion~\cite{tevet2022human} and transformer-based modeling~\cite{guo2022generating,guo2024momask}.
These models can synthesize diverse and realistic motion sequences from natural language prompts, enabling applications in character animation, AR/VR, embodied agents, and human--computer interaction.

A key catalyst behind recent progress is the emergence of large-scale motion--language datasets.
Early benchmarks such as KIT-ML~\cite{Plappert2016} and HumanML3D~\cite{guo2022generating} established the standard text-to-motion setting with paired clip-level captions and standardized evaluation.
However, these captions are typically short and describe an entire motion clip coarsely, limiting fine-grained controllability and semantic grounding. More recent datasets~\cite{han2024amd,wang2024scaling,punnakkal2021babel,xudense,li2026frankenmotion} broaden the supervision spectrum with richer long-form descriptions, hierarchical text annotations, frame-level action labels, and temporally localized captions. Together, these developments highlight a growing demand for \emph{fine-grained temporal correspondence} between motion and language, especially for long-form narratives and multi-action descriptions.
\begin{figure*}[h]
    \centering
\includegraphics[width=\linewidth]{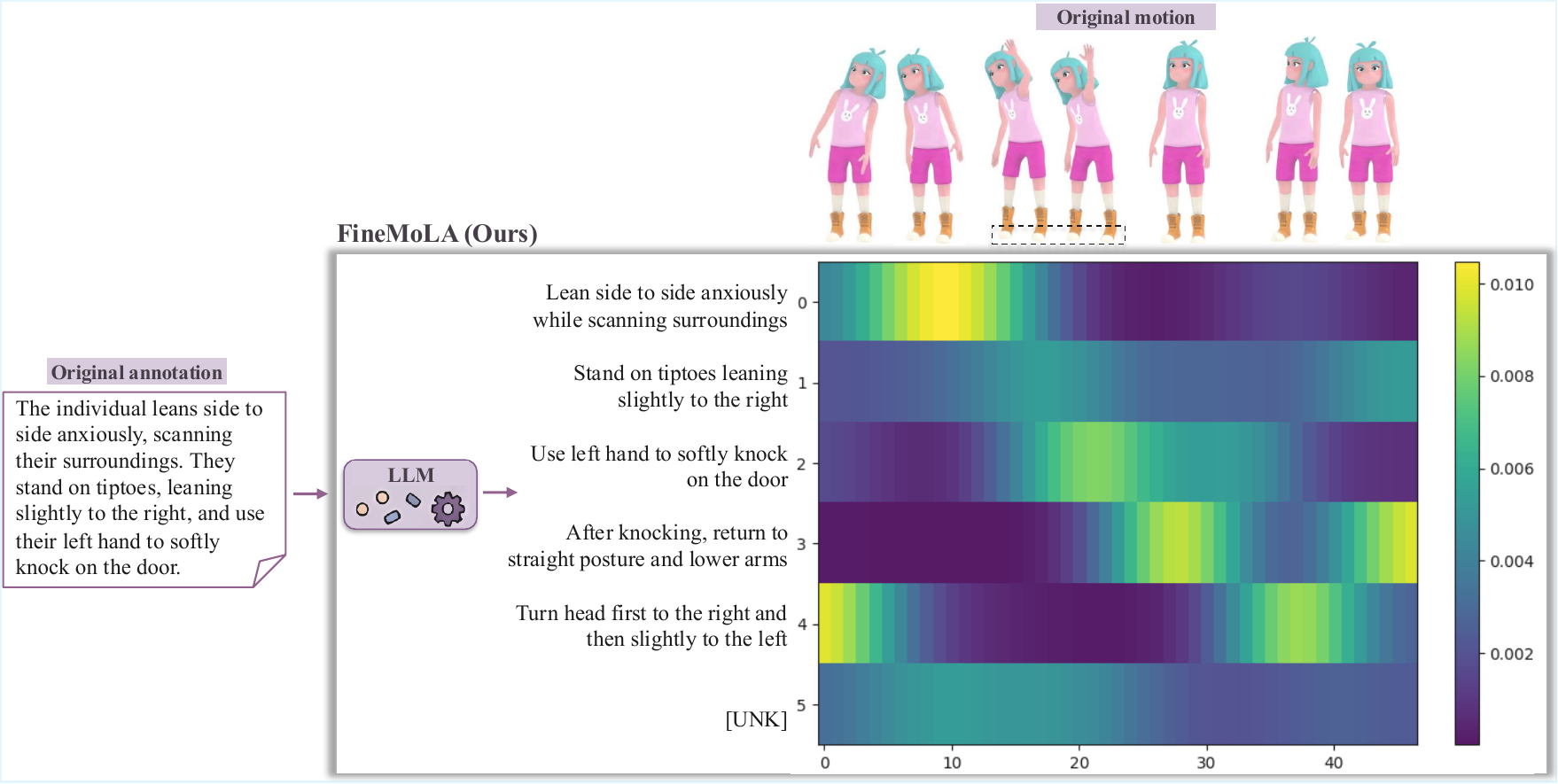}
    \caption{While the SnapMoGen dataset provides human motions with long-form text annotations, \ours finds the fine-grained alignment matrix between motion frames and LLM-detected action-bearing phrases in the text without manual labeling. Our method includes an \texttt{[UNK]} token to absorb frames that are not explicitly described by any specific phrase. The dashed black box highlights that the character does stand on tiptoes in those frames.}
    \label{fig:teaser}
\end{figure*}

Among existing datasets, SnapMoGen~\cite{guo2025snapmogen} stands out for its high-quality motion–language supervision. Its motions are collected using professional mocap suits, and its annotations are primarily written by human annotators. In contrast to earlier datasets that rely on short action labels or simple captions, SnapMoGen provides long-form descriptions that often cover multiple actions and transitions within a single motion clip. Such expressive supervision offers a strong foundation for studying fine-grained motion–text correspondence. However, the annotations remain at the clip level, without explicit frame-to-token alignment between motion and language. Learning such fine-grained correspondence could benefit motion understanding tasks such as temporal localization and dense motion captioning, and provide stronger supervision for fine-grained text-to-motion generation with temporal structure~\cite{wu2025mg,xudense,zhang2023finemogen}. At the same time, obtaining dense open-vocabulary alignment at scale remains challenging due to the inherently many-to-many relationships between words and motion over time.

As illustrated in \cref{fig:teaser}, we aim to recover fine-grained frame--phrase correspondence from long-form clip-level annotations without requiring manual temporal labels. To bridge this gap, we propose a weakly supervised framework that formulates frame--token correspondence as an optimal transport (OT) problem \cite{cuturi2013sinkhorn,peyre2019computational}.
Optimal transport naturally models many-to-many soft alignments under global marginal constraints, making it well-suited for motion-text grounding. To enable efficient and differentiable learning, we adopt entropic regularization and solve the resulting transport plan using Sinkhorn iterations \cite{cuturi2013sinkhorn}.
This formulation aligns with a broader trend in multimodal learning where fine-grained token-level alignment objectives improve cross-modal consistency and robustness \cite{yao2021filip,zhang2025pre,linmulti}.
Leveraging abundant clip-level supervision in SnapMoGen, our method produces pseudo frame-level alignments that support temporally grounded motion--language modeling.

To summarize, our contributions are as follows:
\begin{itemize}
\item We introduce \ours, a weakly supervised framework that learns fine-grained motion--language correspondence from clip-level annotations. By formulating frame-to-token alignments as an optimal transport problem, our method learns fine-grained correspondences for long-form textual descriptions, enabling dense motion--text grounding without requiring manual frame-level labels.
\item Experiments on the SnapMoGen dataset demonstrate that the learned alignments provide stronger motion--text grounding than baselines.
\end{itemize}

\section{Related Work}
\subsection{Human Motion Datasets with Text Annotation}

The rapid progress of text-driven human motion modeling has been closely tied to the emergence of motion datasets with language annotation~\cite{bensabath2024cross,wu2025mg,wang2024scaling,guo2025snapmogen}. Early benchmarks such as KIT-ML~\cite{Plappert2016} and HumanML3D~\cite{guo2022generating} established the standard text-to-motion setting by providing paired motion clips and clip-level natural language descriptions, enabling shared embedding learning, retrieval-based training, and standardized evaluation~\cite{petrovich2023tmr}. 

Subsequent datasets broadened the supervision spectrum beyond coarse clip-level captions. BABEL~\cite{punnakkal2021babel} augments mocap sequences with sequence- and frame-level action labels, supporting temporally localized understanding. Motion-X~\cite{lin2023motion} provides expressive whole-body motion together with sequence-level semantics and video-derived frame-level pose descriptions. More recently, SnapMoGen~\cite{guo2025snapmogen} introduces long-form narrative descriptions with improved temporal continuity, while MotionLib~\cite{wang2024scaling} scales motion--text supervision to millions of samples with hierarchical annotations. Dense Motion Captioning further introduces CompMo, a large-scale dataset with temporally bounded multi-action sequences and grounded captions~\cite{xudense}. In parallel, FrankenMotion provides temporally aware part-level prompts for compositional motion modeling~\cite{li2026frankenmotion}.

Despite these advances, most motion--language datasets still provide supervision either at the clip level or through limited predefined labels, leaving many-to-many frame-level (or token-to-frame) correspondence under-explored. 

This gap motivates learning objectives that explicitly model fine-grained temporal correspondence between motion and language~\cite{xudense,li2025unimotion}, complementary to efforts that focus on stronger generators or clip-level supervision~\cite{tevet2022human,zhang2023generating,guo2024momask,meng2025rethinking}.

\subsection{Human Motion Understanding}
Human motion understanding has long been studied in computer vision through video-based tasks such as action recognition and temporal localization, where models align motion with action labels or textual descriptions~\cite{bertasius2021space,arnab2021vivit,nag2022proposal,feichtenhofer2022masked}. In contrast, our work focuses on structured human motion data (e.g., mocap sequences), where the temporal evolution of poses is directly available and can be paired with language supervision.

In the motion domain, early understanding models largely relied on shared embedding learning and retrieval-style objectives. MotionCLIP~\cite{tevet2022motionclip} aligns motion latents with the CLIP~\cite{radford2021learning} embedding space to exploit strong language priors, while TMR~\cite{petrovich2023tmr} combines contrastive learning with a motion synthesis objective for text--motion retrieval. Cross-dataset studies~\cite{bensabath2024cross} further show that such retrieval performance is sensitive to shifts in motion distributions and annotation styles, motivating robustness and transfer-aware evaluation.

More recent works move beyond global alignment toward fine-grained motion--language understanding. Dense Motion Captioning~\cite{xudense} requires temporally localizing and describing multiple actions within long motion sequences, while unified motion--language models such as MG-MotionLLM~\cite{wu2025mg} and UniMotion~\cite{li2025unimotion} further connect motion understanding, localization, and generation. At the same time, advances in generation and control, including ReMoDiffuse~\cite{zhang2023remodiffuse}, FineMoGen~\cite{zhang2023finemogen}, CoMo~\cite{huang2024controllable}, and recent text-conditioned generators~\cite{zhang2023generating,guo2024momask,meng2025rethinking,Zhao:DartControl:2025,barquero2024seamless,sun2024coma,jiang2023motiongpt,pinyoanuntapong2024mmm,dang2026segmo}, highlight the growing demand for fine-grained temporal correspondence between motion and language. Our work addresses this need by learning dense motion--text alignment from clip-level supervision.

\subsection{Semi-supervised Learning Using Optimal Transport}
Optimal transport (OT) provides a principled tool to establish soft correspondences between two sets of elements under marginal constraints, and has been widely used for matching and alignment problems~\cite{solomon2018optimal,peyre2019computational,cuturi2013sinkhorn}. Recent work has incorporated OT into representation learning and robust alignment.
For instance, OT-based matching has been used to mitigate data poisoning in CLIP-style~\cite{radford2021learning} pre-training by aligning samples through transport plans~\cite{zhang2025pre}.
OT has also been applied to understand and generalize CLIP by modeling relationships between sets of features via transport~\cite{shi2024ot, yao2021filip}. In this work, we focus on \emph{fine-grained correspondences} in the motion-language domain.

\section{Preliminary}
\subsection{Optimal Transport}
Optimal Transport (OT)~\cite{solomon2018optimal} provides a principled framework for measuring the
discrepancy between two distributions while explicitly accounting for the
underlying geometry. Assume two uniform discrete distributions $\mathbf{a}$ and $\mathbf{b}$, supported on $T_m$ and $T_t$ elements, respectively. We denote the set of admissible transport plans by
\begin{equation}
\label{eq:P_constraint}
U(\mathbf{a},\mathbf{b})=
\left\{
\mathbf{P} \in \mathbb{R}_+^{T_m \times T_t}
\;\middle|\;
\mathbf{P}\mathbf{1}_{T_t}=\mathbf{a},\;
\mathbf{P}^\top \mathbf{1}_{T_m}=\mathbf{b}
\right\},
\end{equation}
where
\begin{equation}
\mathbf{a}=\frac{1}{T_m}\mathbf{1}_{T_m},
\qquad
\mathbf{b}=\frac{1}{T_t}\mathbf{1}_{T_t}.
\end{equation}

The marginal constraints ensure that $\mathbf{P}$ redistributes the mass of $\mathbf{a}$ onto $\mathbf{b}$.
Given a cost matrix $\mathbf{C}\in\mathbb{R}^{T_m\times T_t}$ measuring the cost of
transporting unit mass between elements, the Kantorovich formulation of OT is
defined as
\begin{equation}
\label{eq:ot_kantorovich}
\min_{\mathbf{P}\in{U}(\mathbf{a},\mathbf{b})}
\ \langle \mathbf{P},\mathbf{C}\rangle
\end{equation}
where $\langle \mathbf{P},\mathbf{C}\rangle=\sum_{i,j}\mathbf{P}_{ij}\mathbf{C}_{ij}$.

In contrast to pointwise matching, OT naturally supports many-to-many
correspondence and enforces global consistency via marginal constraints.
These properties make OT particularly suitable for modeling fine-grained
alignment under weak supervision, such as motion--text correspondence.

\subsection{Entropic-Regularized Optimal Transport and Sinkhorn Algorithm}
\label{sec:entro}

Directly solving \cref{eq:ot_kantorovich} requires linear programming, which is
computationally expensive and unsuitable for end-to-end learning.
To address this issue, entropic regularization augments the OT objective with an
entropy term:
\begin{equation}
\label{eq:entropic_ot}
\begin{aligned}
\mathbf{P}^* = \arg\min_{\mathbf{P}\in{U}(\mathbf{a},\mathbf{b})}
\ & \langle \mathbf{P},\mathbf{C}\rangle
\;+\;
\varepsilon
\sum_{i,j}{P}_{ij}\big(\log{P}_{ij}-1\big)
\end{aligned}
\end{equation}
where $\varepsilon>0$ controls the strength of the regularization. The entropic regularization renders the problem strictly convex and enables efficient optimization via the Sinkhorn algorithm.

\section{Method}
Suppose we are given a paired motion--text dataset $\mathcal{D}=\{(m^{(j)}, t^{(j)})\}_{j=1}^{N}$.
For simplicity, we omit the sample index $(j)$ when there is no ambiguity.
Each sample consists of a motion sequence $m \in \mathbb{R}^{T_m' \times D_m}$ and a text sequence $t \in \mathbb{R}^{T_t' \times D_t}$. Our goal is to learn fine-grained motion--text alignment.
We model the alignment between a motion sequence and a text sequence as an optimal transport plan
$\mathbf{P}$,
where each entry $P_{uv}$ measures the semantic correspondence between the $u$-th motion frame and the $v$-th text token. We define the transport plan in \cref{sec:ot} and then provide an intuitive analysis of this formulation in \cref{sec:case}. Finally, we describe the feature extraction process in \cref{sec:network} and the CLIP-style contrastive training objective in \cref{sec:training}.

\subsection{Optimal Transport Cost}
\label{sec:ot}
Our model consists of a motion feature extractor $E_m$ and a text feature extractor $E_t$.
Given a motion sequence $m$ and a text sequence $t$, $E_m$ produces a sequence of fine-grained motion features
\begin{equation}
Z^m=\{z^m_1,\ldots,z^m_{T_m}\},
\end{equation}
and $E_t$ produces a sequence of text features
\begin{equation}
Z^t=\{z^t_1,\ldots,z^t_{T_t}\}.
\end{equation}
Note that $T_m\neq T_m'$ and $T_t\neq T_t'$ due to down-sampling in the encoders. Let $s(\cdot,\cdot)$ denote the cosine similarity between $\ell_2$-normalized features.

We define the transport cost matrix $\mathbf{C}\in\mathbb{R}^{T_m\times T_t}$ by
\begin{equation}
\label{eq:cost_matrix}
\mathbf{C}_{uv}=1-s(z^m_u,z^t_v).
\end{equation}
Intuitively, a smaller cost indicates a stronger semantic correspondence between the $u$-th motion frame and the $v$-th text token.

Based on the admissible transport set $U(\mathbf{a},\mathbf{b})$ defined in the previous subsection, we define the OT cost for a motion--text pair $(m,t)$ as
\begin{equation}
\label{eq:ot_single}
\mathcal{J}_{\mathrm{OT}}(m,t)
=
\min_{\mathbf{P}\in U(\mathbf{a},\mathbf{b})}
\langle \mathbf{P},\mathbf{C}\rangle.
\end{equation}

This OT cost measures the minimum transport cost required to align the motion and text features under the marginal constraints.
A lower value of $\mathcal{J}_{\mathrm{OT}}(m,t)$ indicates better fine-grained motion--text alignment. As in \cref{sec:entro}, we follow \cref{eq:entropic_ot} to obtain approximated optimal transport plan $\mathbf{P}'$. The estimated optimal transport cost is 
\begin{equation}
\label{eq:eot}
\mathcal{J}_{\mathrm{EOT}} = \langle \mathbf{P}',\mathbf{C}\rangle
\end{equation}
Although $\mathbf{P}'$ is a function of $\mathbf{C}$, we treat it as a constant during backpropagation for simplicity. Under this approximation, $\mathcal{J}_{\mathrm{EOT}}$ remains differentiable with respect to $\mathbf{C}$.

\subsection{Analysis: a Case Study}
\label{sec:case}
To better understand how the proposed objective encourages fine-grained alignment, we consider an idealized case in which each motion frame is described by exactly one text token.

\begin{assumption}[Perfect Frame--Token Correspondence]
\label{asmp:perfect_alignment}
Assume that the motion and text sequences have equal lengths, i.e., $T_m=T_t$.
Further assume that there exist encoder parameters $\theta_m$ and $\theta_t$ such that, for each paired sample $(m,t)$, there exists a bijection
$\pi:\{1,\ldots,T_m\}\rightarrow\{1,\ldots,T_t\}$ satisfying
\begin{equation}
s\!\left(z^m_u,z^t_{\pi(u)}\right)=1,
\quad \forall u\in\{1,\ldots,T_m\},
\end{equation}
where $z^m_u$ and $z^t_v$ denote the $u$-th and $v$-th output features of $E_m(m;\theta_m)$ and $E_t(t;\theta_t)$, respectively.
\end{assumption}

\begin{theorem}[Zero OT Cost under Perfect Alignment]
\label{thm:zero_ot}
Under Assumption~\ref{asmp:perfect_alignment}, the OT objective $\mathcal{J}_{OT}(m,t)$ admits a feasible transport plan with zero cost. Specifically, the transport plan defined by
\begin{equation}
P_{uv}=\frac{1}{T_m}\mathbbm{1}\{v=\pi(u)\}
\end{equation}
is feasible and achieves
\begin{equation}
\mathcal{J}_{OT}(m,t)=0.
\end{equation}
\end{theorem}

\begin{proof} See supplementary.
\end{proof}

In practice, motion-text alignment is generally many-to-many rather than bijective. Nevertheless, \cref{thm:zero_ot} shows that the proposed objective is consistent with perfect fine-grained alignment in an idealized setting, which helps explain why minimizing $\mathcal{J}_{\mathrm{OT}}$ can encourage meaningful motion--text correspondence in practice.

\begin{algorithm}
\caption{CLIP-style Motion--Text Training with Sinkhorn-based Alignment}
\label{alg:ot_motion_text}
\begin{algorithmic}[1]
\REQUIRE Paired dataset $\mathcal{D}=\{(m^{(j)},t^{(j)})\}_{j=1}^N$;
motion feature extractor $E_m(\cdot;\theta_m)$; text feature extractor $E_t(\cdot;\theta_t)$
\ENSURE Trained parameters $\theta_m,\theta_t$

\FOR{each training epoch}
    \FOR{each mini-batch $\{(m^{(i)},t^{(i)})\}_{i=1}^B$}
        \FOR{$i=1$ to $B$}
            \STATE $Z^{m,(i)} \leftarrow E_m(m^{(i)};\theta_m)$
            \STATE $Z^{t,(i)} \leftarrow E_t(t^{(i)};\theta_t)$
            \STATE $\ell_2$-normalize all feature vectors in $Z^{m,(i)}$ and $Z^{t,(i)}$
        \ENDFOR

        \FOR{$i=1$ to $B$}
            \FOR{$j=1$ to $B$}
                \STATE \textcolor{green!50!black}{Construct the cost matrix $\mathbf{C}^{(i,j)}$ from $Z^{m,(i)}$ and $Z^{t,(j)}$ according to \cref{eq:cost_matrix}}
                \STATE \textcolor{green!50!black}{Compute the Sinkhorn transport plan $P^{(i,j)} \in U(\mathbf{a},\mathbf{b})$ according to \cref{eq:entropic_ot}.} 
                \STATE \textcolor{green!50!black}{Compute the pairwise alignment logit $L_{ij} = -\mathcal{J}_{\mathrm{EOT}}(m^{(i)}, t^{(j)})$ using \cref{eq:ot_single}}
            \ENDFOR
        \ENDFOR

        \STATE Compute motion-to-text probabilities $p(t^{(j)}\mid m^{(i)})$ by row-wise softmax over $L$
        \STATE Compute text-to-motion probabilities $p(m^{(i)}\mid t^{(j)})$ by column-wise softmax over $L$
        \STATE Compute the symmetric contrastive loss $\mathcal{L}=\frac{1}{2}(\mathcal{L}_{m2t}+\mathcal{L}_{t2m})$
        \STATE Update $\theta_m,\theta_t$ by minimizing $\mathcal{L}$
    \ENDFOR
\ENDFOR
\end{algorithmic}
\end{algorithm}

\begin{figure*}
    \centering
    \includegraphics[width=0.95\linewidth]{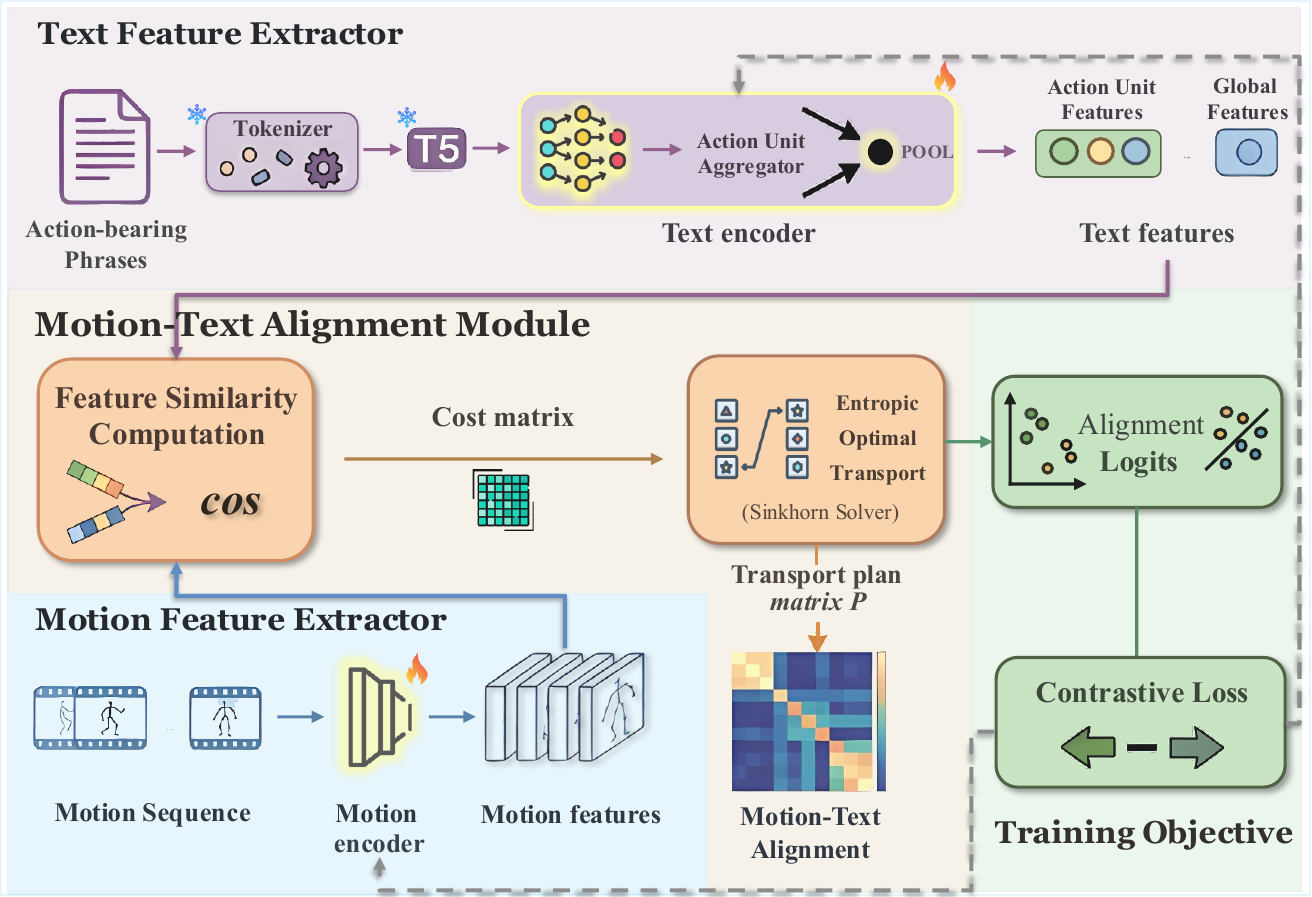}
    \caption{Overview of \ours. The text stream first encodes the action-bearing phrases with a frozen T5 backbone. Then the action unit aggregator aggregates the features into action unit features and the global feature, which form a text feature set. In parallel, the motion stream extracts temporal motion features using a temporal-convolution-based encoder. Given the motion and text features, we compute a feature-similarity-based cost matrix and solve an entropic optimal transport problem with the Sinkhorn algorithm to obtain pairwise alignment logits. These logits are then used in a CLIP-style symmetric contrastive objective for training.}
    \label{fig:pipeline}
    \vspace{-3mm}
\end{figure*}

\subsection{Feature Extraction}
\label{sec:network}
The weakly supervised learning framework is illustrated in \cref{fig:pipeline}. Our architecture adopts a dual-stream design to learn representations from textual descriptions and motion sequences. The original expressive text annotation is first segmented into multiple \textbf{action-bearing phrases} using an LLM (details are in the supplementary material). These phrases are then concatenated into a sequence and fed into the text feature extractor. The text feature extractor $E_t$ leverages a frozen T5-base~\cite{2020t5} backbone to extract latent semantic features, which are subsequently projected into a latent space through an additional MLP. To facilitate temporal alignment, we introduce an \textbf{action unit aggregator} that groups tokens belonging to the same phrase and pools their embeddings into compact \textbf{action unit features}. Some motion segments may not be explicitly described by any action-bearing phrase in the text. To avoid imposing spurious correspondences, we do not force every motion frame to align with an action-unit representation alone and allow a motion frame to correspond to an \texttt{[UNK]} token. We introduce a \textbf{global feature} mechanism that aggregates sequence-level information and provides holistic semantic context for the contrastive objective. This global feature is appended to the action unit features to form the \textbf{text features} $Z^t$ in \cref{alg:ot_motion_text}. On the motion stream $E_m$, we finetune the pre-trained VQ-VAE encoder~\cite{guo2025snapmogen}, composed of temporal convolutional layers, to extract \textbf{motion features} $Z^m$ from motion sequences.

\subsection{Training Objective}
\label{sec:training}
As illustrated in \cref{alg:ot_motion_text}, the training process is conceptually similar to CLIP~\cite{radford2021learning}, in that both methods learn cross-modal representations through batch-wise contrastive supervision over paired motion--text samples. We highlight in green the steps that differ from standard CLIP training.

Given a mini-batch of paired samples
$\mathcal{D}$,
CLIP first computes a global representation for each sample in the two modalities and then forms a batch-wise similarity matrix, where the diagonal entries correspond to matched pairs and the off-diagonal entries correspond to mismatched pairs. A symmetric cross-entropy objective is subsequently applied to encourage high similarity for ground-truth pairs and low similarity for mismatched pairs.

Our training pipeline follows the same high-level paradigm, but differs in how the pairwise alignment score is computed. Instead of directly measuring similarity between two global embeddings, we first compute a fine-grained motion--text alignment cost using the entropy-regularized OT objective in \cref{eq:entropic_ot}. Specifically, for each pair $(m^{(i)}, t^{(j)})$ in the batch, we compute
\begin{equation}
\label{eq:pairwise_eot}
L_{ij} = -\mathcal{J}_{\mathrm{EOT}}(m^{(i)}, t^{(j)}),
\end{equation}
where a larger $L_{ij}$ indicates stronger cross-modal alignment.

Using these pairwise alignment logits, we define the motion-to-text matching probability as
\begin{equation}
p(t^{(j)} \mid m^{(i)})
=
\frac{\exp(L_{ij}/\tau)}
{\sum_{k=1}^{B}\exp(L_{ik}/\tau)},
\end{equation}
and analogously define the text-to-motion matching probability as
\begin{equation}
p(m^{(i)} \mid t^{(j)})
=
\frac{\exp(L_{ij}/\tau)}
{\sum_{k=1}^{B}\exp(L_{kj}/\tau)},
\end{equation}
where $\tau$ is a temperature parameter and $B$ is the batch size.

We then optimize the model using a symmetric contrastive loss:
\begin{equation}
\mathcal{L}_{m2t}
=
-\frac{1}{B}\sum_{i=1}^{B}
\log p(t^{(i)} \mid m^{(i)}),
\end{equation}
\begin{equation}
\mathcal{L}_{t2m}
=
-\frac{1}{B}\sum_{i=1}^{B}
\log p(m^{(i)} \mid t^{(i)}),
\end{equation}
and the final training objective is
\begin{equation}
\label{eq:final_training_loss}
\mathcal{L}
=
\frac{1}{2}\left(\mathcal{L}_{m2t}+\mathcal{L}_{t2m}\right).
\end{equation}

Compared with CLIP, the key difference is that our pairwise logits are not produced by a single global embedding similarity. Instead, they are derived from a fine-grained Sinkhorn-based transport cost between motion frames and text units, allowing the model to capture local temporal correspondence while still benefiting from batch-wise contrastive supervision.

\section{Experiments}

\begin{figure*}
    \centering
    \includegraphics[width=1.05\linewidth]{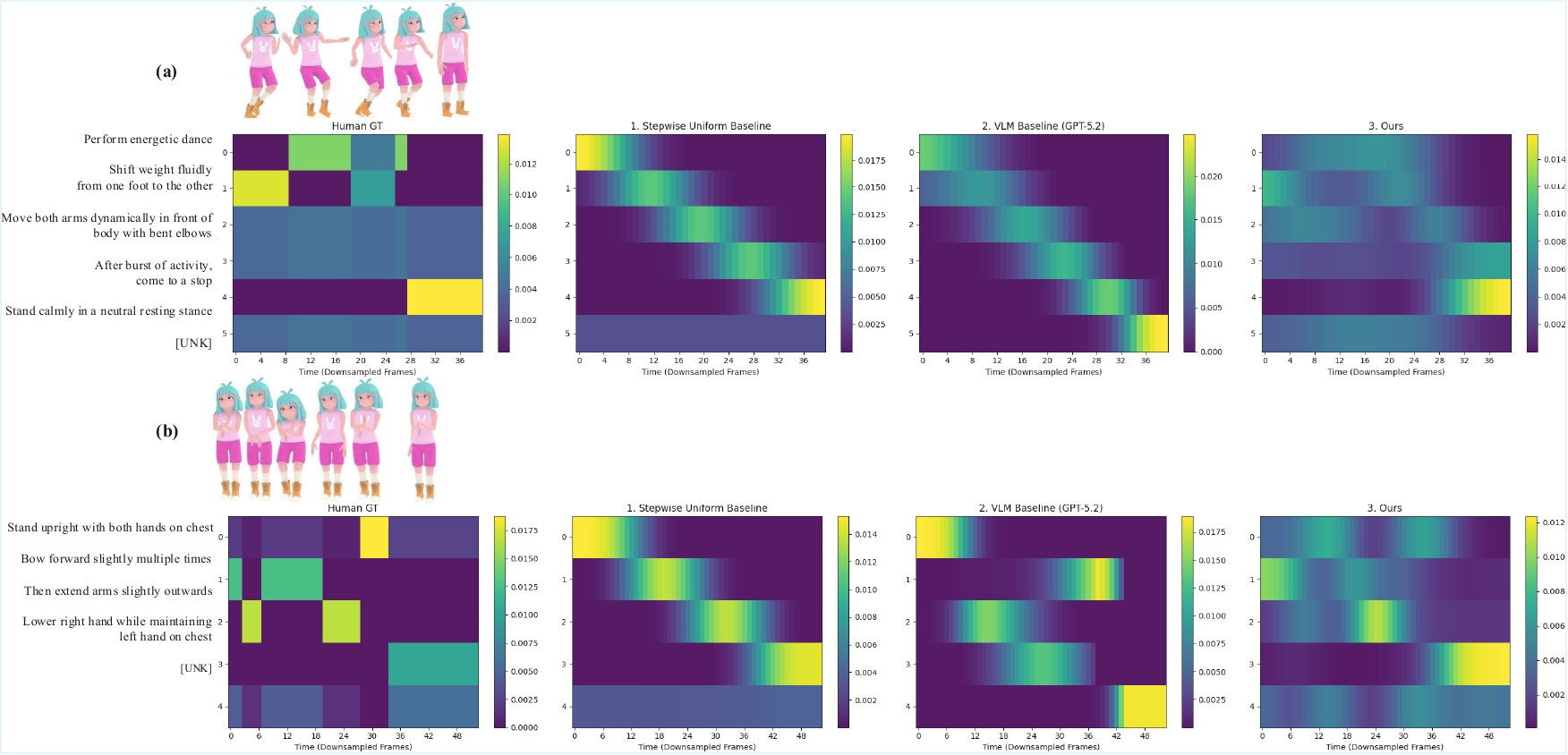}
    \caption{Qualitative comparison with baselines. We plot the ground truth transport plan and those produced by different methods. Our method produces transport plans that are structurally similar to the ground truth.}
    \label{fig:qual1}
    \vspace{-5mm}
\end{figure*}
\subsection{Setting and Implementation Details}
\paragraph{Dataset.} We conduct our experiments on the SnapMoGen dataset~\cite{guo2025snapmogen}, a comprehensive collection of high-fidelity human motion capture data paired with rich textual descriptions. The dataset comprises 34,750 training sequences and 2,012 validation sequences. All motions are originally recorded at 30 FPS. Following SnapMoGen~\cite{guo2025snapmogen}, we downsample the sequences by a factor of 4, resulting in an effective frame rate of 7.5 FPS for both the motion embeddings and the alignment computation. We choose a set of $30$ motion-text pairs from the test set and manually label the ground truth alignment matrix $\mathbf{P}^{gt} \in \mathbb{R}^{T_m \times T_t}$, where $T_m$ and $T_t$ denote the number of downsampled motion frames and atomic action segments, respectively. These manual annotations are used solely for quantitative evaluation and are never used during training. The details of the labeling process are deferred to the supplementary.

\paragraph{Evaluation Metrics.} \label{paragraph:eval_metric}
We evaluate the alignment accuracy by measuring the discrepancy between the transport matrix $\mathbf{P}$ and the ground-truth annotations $\mathbf{P}^{gt}$. Specifically, we compute the normalized $\ell_1$ distance between the two matrices: 
\begin{equation}    
d(\mathbf{P}, \mathbf{P}^{gt}) = \frac{1}{T_m \times T_t} \sum_{i=1}^{T_m} \sum_{j=1}^{T_t} \left| {\mathbf{P}}_{ij} - {\mathbf{P}}^{gt}_{ij} \right|
\end{equation}

where both matrices are first normalized via the Sinkhorn-Knopp algorithm to enforce doubly stochastic constraints. This ensures they satisfy the predefined row and column marginals (i.e., $\sum_j {\mathbf{P}}_{ij} = 1/T_m$ and $\sum_i {\mathbf{P}}_{ij} = 1/T_t$).

\paragraph{Implementation details.}
Our framework is implemented in PyTorch and trained on a single NVIDIA L40 GPU. Training for 500 epochs takes approximately 12 hours. We use AdamW as the optimizer with an initial learning rate of $2 \times 10^{-4}$ and a batch size of 128. For text encoding, we adopt a frozen T5-base backbone followed by five trainable residual MLP blocks, which project the text features into a 512-dimensional embedding space. For Sinkhorn-based alignment, we set the entropy regularization parameter to $\varepsilon=0.1$ and use 50 forward iterations, which we find sufficient for stable convergence in practice. The contrastive temperature is set to $\tau=0.07$. In practice, $T_m$ is typically around 50, while $T_t$ is typically around 5.\\

\subsection{Comparison with Baselines} 
\paragraph{Stepwise Uniform Baseline.} To demonstrate the necessity of dynamic alignment, we construct a stepwise uniform baseline, which inherently assumes a naive, sequential progression between textual actions and motion segments. Specifically, given a motion sequence of $T_m$ frames and $N$ textual actions, this baseline constructs an initial binary matrix by uniformly allocating exactly $\lfloor T_m / N \rfloor$ consecutive frames to each action index in chronological order. Note that the resulting matrix is not strictly block-diagonal; applying the Sinkhorn-Knopp algorithm enforces marginal distribution constraints, which naturally softens the hard boundaries and diffuses the alignment probabilities toward adjacent frames.

\paragraph{VLM Baseline.} To establish a competitive baseline against state-of-the-art vision-language reasoning, we leverage GPT-5.2~\cite{singh2025openai} as a zero-shot temporal segmenter. We provide rendered frames sampled at $2.0$ FPS and segmented action-bearing phrases to the VLM and ask it to predict the alignment. Details are provided in the supplementary material.

\begin{table*}[t]
\centering
\caption{Quantitative results. Values are reported in normalized $\ell_1$ distance ($\times 10^{-2}$), where lower is better. GF indicates whether to use the global feature. MLP Enc. determines whether to use an MLP-based text encoder or a self-attention-based text encoder.} 
\label{tab:alignment_results}
\begin{tabular}{lcccc}
\toprule
Method & GF & MLP Enc. & temperature & Normalized $\ell_1 \downarrow$ \\
\midrule

Stepwise Uniform Baseline & & & & 0.347 \\
VLM Baseline (GPT-5.2) & & & & 0.296 \\

\midrule
Ours (GF + MLP encoder) & \checkmark & \checkmark & 0.07 & \textbf{0.225} \\
Ours (GF + MLP encoder) & \checkmark & \checkmark& 0.01 & 0.230 \\
Ours (GF + MLP encoder) & \checkmark & \checkmark& 0.03 & 0.237 \\
Ours (GF + attention encoder) & \checkmark  & $\times$ & 0.07 & 0.275 \\
Ours (MLP encoder) & $\times$ & \checkmark & 0.07& 0.239 \\

\bottomrule
\end{tabular}
\vspace{-2mm}
\end{table*}





\paragraph{Results.} \Cref{tab:alignment_results} presents the quantitative evaluation of our proposed model. The results demonstrate that \ours achieves a significant relative error reduction over the VLM and stepwise uniform baselines. These results demonstrate the necessity of our dedicated data-driven alignment framework.

\subsection{Ablation} 
\paragraph{Impact of the Global Feature Mechanism.} 
A key component of our alignment module is the global feature, which acts as an \texttt{[UNK]} token to absorb transitional or semantically ambiguous frames. In practice, frames near action boundaries often do not correspond strongly to any specific text token. Without such a token, optimal transport tends to force these frames onto neighboring action tokens, leading to blurred boundaries. During evaluation, we reassign the accumulated probability of the \texttt{[UNK]} token strictly to the preceding valid action token. By introducing the global feature, our model can naturally capture these uncertain frames and preserve cleaner temporal segmentation. As shown in \Cref{tab:alignment_results}, this design consistently improves alignment performance over the variant without a global token.

\paragraph{MLP vs. Attention-based Text Encoders.}
As shown in \cref{tab:alignment_results}, replacing the lightweight MLP with a self-attention-based encoder increases the error. We conjecture that self-attention mixes excessive global context into each action token, blurring semantic boundaries and weakening temporal correspondence. In contrast, the MLP better preserves local phrase-level semantics for precise alignment.

\paragraph{Sensitivity Analysis of Entropic Regularization}
To evaluate the robustness of OT alignment, we conduct a sensitivity analysis on the entropic regularization parameter $\varepsilon$. In the entropic OT framework, $\varepsilon$ controls the trade-off between transport cost minimization and the sparsity of the alignment matrix $P$. As shown in Figure \ref{fig:sensitivity_entropy}, we monitor two primary convergence metrics: 
(1) Marginal Error,
\begin{equation}
E_{marg} = \sum_{i} \left| \sum_{j} P_{ij} - a_i \right| + \sum_{j} \left| \sum_{i} P_{ij} - b_j \right|
\end{equation}
which quantifies the satisfaction of the double-stochasticity constraints in \cref{eq:P_constraint},
and (2) Delta P,
\begin{equation}
\Delta P^{(k)} = || P^{(k)} - P^{(k-1)} ||_1 = \sum_{i} \sum_{j} \left| P_{ij}^{(k)} - P_{ij}^{(k-1)} \right|
\end{equation}
which measures the stability of the transport plan between successive iterations. Our results indicate that setting $\varepsilon = 0.01$ significantly increases the number of iterations required to reach numerical stability. Conversely, $\varepsilon = 0.1$ achieves a balance between alignment precision and training efficiency. We therefore use $50$ iterations in all experiments.

\begin{figure}
    \centering
    \includegraphics[width=1\linewidth]{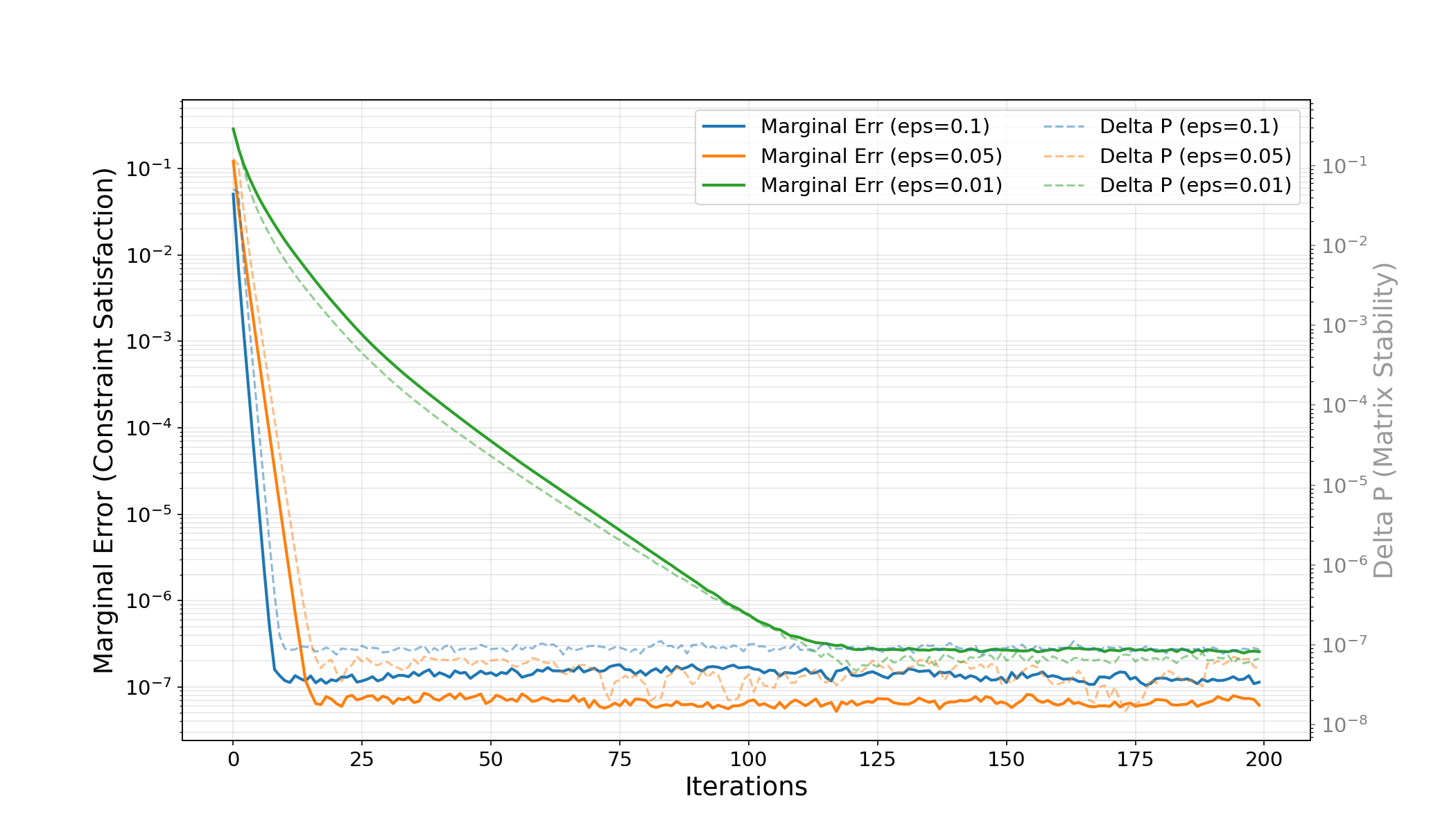}
    \caption{Sinkhorn Convergence Analysis. The plot illustrates the impact of different entropic regularization parameters ($\varepsilon$) on the convergence rates of Marginal Error (representing constraint satisfaction) and Delta P (representing pointwise matrix stability). The results are obtained using a randomly initialized cost matrix $C$.}
    \label{fig:sensitivity_entropy}
    \vspace{-2mm}
\end{figure}

\subsection{Qualitative Results}
Qualitative results of the frame-level text-motion alignment are illustrated in \Cref{fig:qual1}. Compared to the VLM and stepwise uniform baselines, the transport matrix $\mathbf{P}$ generated by our model shows superior structural correspondence with the ground truth $\mathbf{P}^{gt}$. Notably, for captions containing action-bearing phrases that encompass extended temporal spans, our model effectively captures these long-range dependencies. For instance, in sample (a), the first three action phrases in $\mathbf{P}$ demonstrate a relatively uniform probability distribution across the corresponding frames. This highlights the effectiveness of our framework in modeling complex many-to-many alignments between motion and text, a challenging scenario where the VLM baseline falls short.
\vspace{3mm}
\section{Conclusion}
This paper introduced \ours, a weakly supervised approach for recovering fine-grained frame--phrase correspondence from long-form clip-level motion annotations. Our key idea is to cast motion--language alignment as an entropic optimal transport problem, which provides a principled way to model many-to-many correspondence between motion frames and textual units under weak supervision. Combined with LLM-based phrase decomposition and CLIP-style contrastive training, the proposed framework learns temporally grounded motion--text alignment without manual frame-level labels. Experiments show that \ours yields more accurate motion--text grounding than existing baselines. We hope this work encourages future research on fine-grained motion--language supervision.
{
    \small
    \bibliographystyle{ieeenat_fullname}
    \bibliography{main}
}

\clearpage
\setcounter{page}{1}
\maketitlesupplementary

\begin{figure*}[h]
    \centering
    \includegraphics[width=\linewidth]{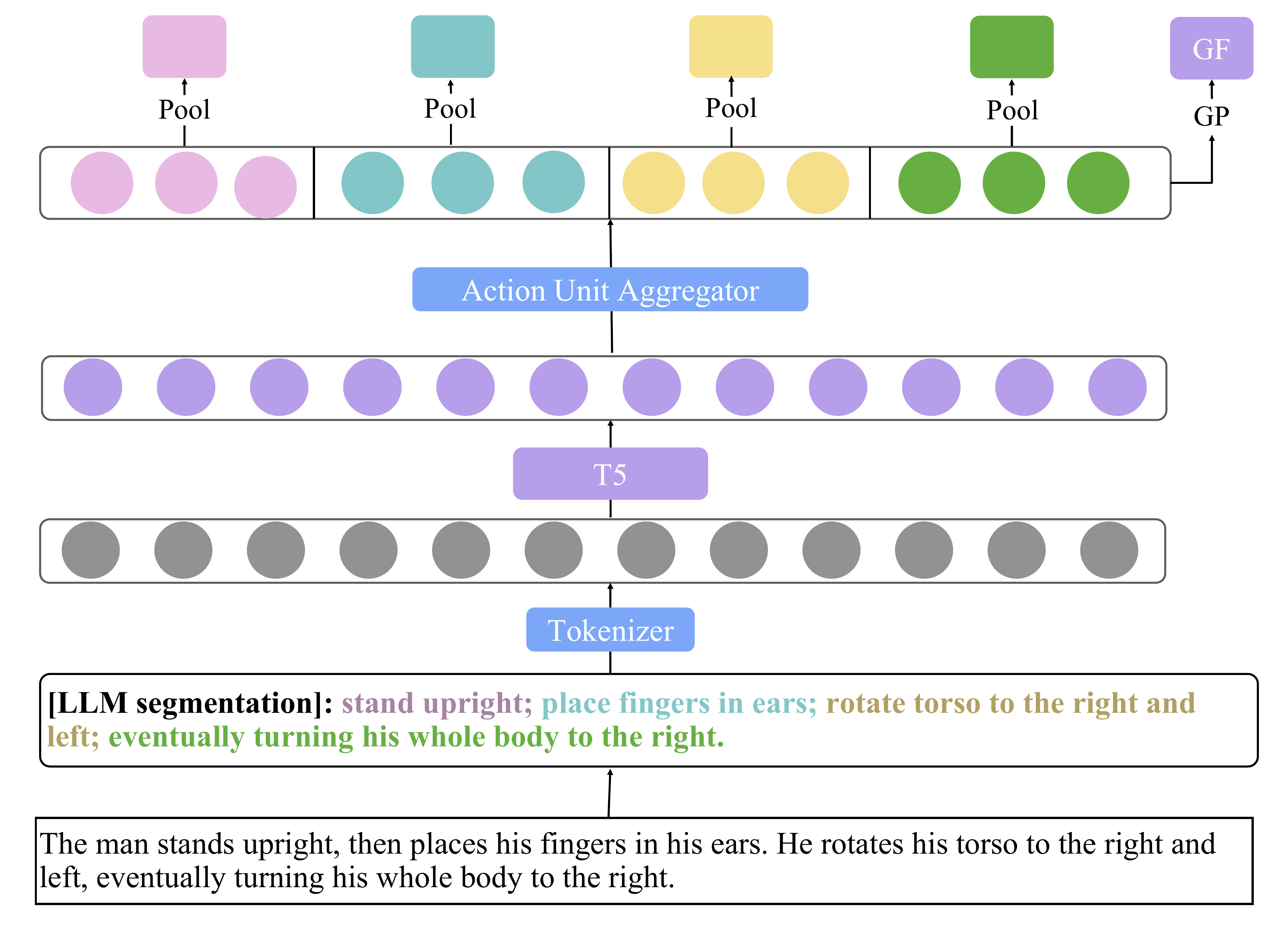}
    \caption{
Overview of the \textbf{action unit aggregator}. 
A long-form motion caption is first segmented into multiple \textbf{action-bearing phrases} using an LLM. 
The concatenated phrase sequence is tokenized and encoded by a frozen T5 backbone to obtain token-level embeddings. 
The action unit aggregator then groups tokens belonging to the same phrase and pools them into compact \textbf{action unit features}. 
An additional \textbf{global feature} (GF) is obtained by global pooling (GP) over the entire sequence, which acts as an \texttt{[UNK]} token to absorb motion frames that are not explicitly described by any phrase.}
    \label{fig:aggre}
\end{figure*}

\section{Proof of Theorem 1}
\begin{theorem}[Zero OT Cost under Perfect Alignment]
\label{thm:zero_ot}
Under Assumption 1, the OT objective $\mathcal{J}_{OT}(m,t)$ admits a feasible transport plan with zero cost. Specifically, the transport plan defined by
\begin{equation}
P_{uv}=\frac{1}{T_m}\mathbbm{1}\{v=\pi(u)\}
\end{equation}
is feasible and achieves
\begin{equation}
\mathcal{J}_{OT}(m,t)=0.
\end{equation}
\end{theorem}

\begin{proof}
Under Assumption 1, for each motion frame $u$, there exists a unique text token $\pi(u)$ such that
$s(z^m_u,z^t_{\pi(u)})=1$.
By Eq. 7, this implies that $\mathbf{C}_{u,\pi(u)}=0$.
Now consider the transport plan
\begin{equation}
P_{uv}=\frac{1}{T_m}\mathbbm{1}\{v=\pi(u)\}.
\end{equation}
Since $\pi$ is a bijection and $T_m=T_t$, each row and each column of $P$ contains exactly one nonzero entry equal to $1/T_m$. Therefore, $P$ satisfies the marginal constraints and is feasible. The total transport cost is
\begin{equation}
\sum_{u=1}^{T_m}\sum_{v=1}^{T_t} P_{uv}\mathbf{C}_{uv}=0,
\end{equation}
which is the minimum possible value. Hence, $\mathcal{J}_{OT}(m,t)=0$.
\end{proof}
\section{Caption Segmentation}
To bridge the gap between long-form narrative descriptions and local temporal motion segments, we first decompose each caption into multiple \textbf{action-bearing phrases} using an LLM-based caption segmentation pipeline, as illustrated in \cref{fig:aggre}. 
Specifically, GPT-4o-mini rewrites each caption into a sequence of semicolon-separated phrases, where each phrase describes a temporally coherent action while preserving the original temporal order. The prompt is illustrated in \cref{fig:segmentation_prompt}.

The segmented phrases are concatenated and processed by a tokenizer and a frozen T5 encoder to obtain token-level embeddings. 
We then apply the \textbf{action unit aggregator}, which identifies token spans corresponding to each phrase using separator tokens (``\texttt{;}'') and the end-of-sequence token. 
For each span, token embeddings are pooled along the sequence dimension to produce compact \textbf{action unit features}.

In addition, we compute a \textbf{global features} by pooling over the entire token sequence. 
This global feature acts as an \texttt{[UNK]} token that absorbs motion frames not explicitly described by any phrase. 
The final text feature set, therefore, consists of the action unit features together with the global feature, which are used for the Sinkhorn-based motion–text alignment described in the main paper.

\begin{figure*}[h]
\begin{mdframed}[backgroundcolor=gray!10, roundcorner=5pt, innertopmargin=10pt, innerbottommargin=10pt, innerrightmargin=10pt, innerleftmargin=10pt]
{\small
\textbf{System Prompt}  \vspace{0.5em}
You are an expert in Motion Capture (MoCap) and computer animation. Your task is to segment complex motion captions into "Atomic Action Units" to facilitate temporal alignment for motion generation models (e.g., SnapMoGen).
 \vspace{0.5em}
 
\textbf{Segmentation Principles:}
\begin{enumerate}
    \item \textbf{Atomization with Temporal Context:} Break down complex sentences into basic, physically continuous atomic actions.
    \item \textbf{Preserve Sequence Markers:} Keep important temporal markers like "then", "soon after", "finally", or "subsequently" at the beginning of descriptions to provide directionality for alignment.
    \item \textbf{Handle Concurrency:} If actions happen simultaneously (e.g., "while", "during"), include concurrency keywords in the description (e.g., "wave right hand while walking") to signify temporal overlap.
    \item \textbf{Body-part Decoupling:} Separate concurrent movements of different body parts if they are distinct actions.
    \item \textbf{Standardization:} Output must be a structured JSON array of objects.
\end{enumerate}  \vspace{0.5em}
\textbf{Output Format Requirement:} Always return a pure JSON object containing a key \texttt{"results"} which is a list of objects, each with \texttt{"caption\_index"} and \texttt{"atomic\_actions"} (list of objects with \texttt{"id"} and \texttt{"description"}).
 
 \vspace{0.5em}
\hrule
 \vspace{0.5em}

\textbf{Few-Shot Examples} \vspace{0.5em}

\textbf{Example Input Captions to Split:}
\begin{itemize}
    \item "The person walks forward a few steps, pauses, and then steps to her left with her left foot, standing still. She raises her arms slightly in front of her at stomach height before lowering them. She repeats the side step with her left foot and turns her head to the left."
    \item "The person stands with hands raised while slowly nodding their head, then starts to move his hands down while tilting his body forward."
    \item "The person advances quickly, bending at the waist during the walk to pick up an object, then finally stands upright while gazing at the object."
\end{itemize}

\textbf{Example Segmentation Results:}
\begin{verbatim}
{
  "results": [
    {
      "caption_index": 1,
      "atomic_actions": [
        {"id": 1, "description": "walk forward a few steps"},
        {"id": 2, "description": "pause and stand still"},
        {"id": 3, "description": "then step to the left with left foot"},
        {"id": 4, "description": "standing still again"},
        {"id": 5, "description": "raise arms slightly to stomach height"},
        {"id": 6, "description": "soon after, lower arms"},
        {"id": 7, "description": "repeat side step with left foot"},
        {"id": 8, "description": "simultaneously turn head to the left"}
      ]
    },
    {
      "caption_index": 2,
      "atomic_actions": [
        {"id": 1, "description": "stand with hands raised"},
        {"id": 2, "description": "slowly nod head while standing"},
        {"id": 3, "description": "then start to move hands down"},
        {"id": 4, "description": "tilt body forward while moving hands"}
      ]
    },
    {
      "caption_index": 3,
      "atomic_actions": [
        {"id": 1, "description": "advance quickly"},
        {"id": 2, "description": "bend at the waist while walking to pick up object"},
        {"id": 3, "description": "then finally stand upright"},
        {"id": 4, "description": "gaze at the object while standing"}
      ]
    }
  ]
}
\end{verbatim}
}
\end{mdframed}
\caption{The complete system prompt and contextual few-shot examples provided to the LLM for atomic action units segmentation.}
\label{fig:segmentation_prompt}
\end{figure*}
\section{Ground Truth Annotation} 
To quantitatively evaluate the temporal alignment performance, we manually label the ground truth alignment matrix $\mathbf{P}^{gt} \in \mathbb{R}^{T_m \times T_t}$. The annotation process involves manual segment-level boundary identification: for each action segment $j$, human annotators specify the temporal interval $[t_{start}, t_{end}]$ in seconds. These intervals are subsequently mapped to discrete frame indices based on the video rendering frame rate. We construct $\mathbf{P}^{gt}$ initially as a binary indicator matrix, where $p_{i,j}=1.0$ if frame $i$ falls within the temporal scope of action $j$. To map it to $U(\mathbf{a},\mathbf{b})$, we apply the Sinkhorn-Knopp algorithm to iteratively normalize the binary matrix. This yields a soft, doubly stochastic probability distribution in $U(\mathbf{a},\mathbf{b})$.

\section{VLM Baseline}
For the VLM baseline, motion videos are first downsampled to a fixed temporal resolution of 2.0~FPS. For motion sequences spanning 30 to 40 seconds, this sampling strategy yields approximately 60 to 80 frames, which provides sufficient visual density to capture macroscopic action transitions. We utilize a structured system prompt that configures the LLM as a ``specialized human motion analyst.'' The model receives the sequence of visual frames alongside the full unsegmented contextual caption and the explicitly parsed action sequence. To mitigate temporal hallucinations, which are prevalent in LLM-based video reasoning, we overlay the original frame indices onto the top-left corner of each sampled image. These indices serve as absolute temporal anchors, enabling the model to ground its predictions in visible spatial markers. Finally, the model outputs the localized start and end frames in a structured JSON format to ensure reliable temporal grounding.

\end{document}